\documentclass[12pt]{article}

\usepackage{times}
\usepackage{amssymb}
\usepackage{amsmath}
\usepackage{epsfig}

\usepackage{algorithm}
\usepackage{algorithmic}
\usepackage{epsfig}
\usepackage{amsmath,amsthm,amsfonts,amssymb,mathrsfs}
\newtheoremstyle{example}
{\topsep} {\topsep}%
{\upshape}
{}
{}
{\newline}
{}

\newcommand{\NN}{{N}}

\newcommand{\II}{{p}}
\newcommand{\subgrad}{{g}}
\newcommand{\tilder}{{k}}

\newcommand{\Fr}{{\rm 2}}
\newcommand{\Sp}{{\rm \infty}}
\newcommand{\N}{{\mathbb N}}
\newcommand{\R}{{\mathbb R}}

\newcommand{\X}{{\boldsymbol{\cal W}}}
\newcommand{\W}{{{\cal W}}} 
\newcommand{\bW}{{\boldsymbol{\cal W}}} 
\newcommand{\bB}{{\boldsymbol{\cal B}}}

\newcommand{\C}{{\cal C}}

\newcommand{\I}{{\cal I}}

\newcommand{\aP}{{\tilde P}}

\newcommand{\beq}{\begin{equation}}
\newcommand{\eeq}{\end{equation}} 
\newcommand{\bea}{\begin{eqnarray}}
\newcommand{\eea}{\end{eqnarray}}

 \newcommand{\lb}{{\langle}}
\newcommand{\rb}{{\rangle}} 

\def\boldf#1{\hbox{\rlap{$#1$}\kern.4pt{$#1$}}}

 \newcommand{\trans}{^{\scriptscriptstyle
\top}}

 \newcommand{\calG}{{\cal G}}
 
\newcommand{\rank}{{\rm rank}}

\newtheorem{theorem}{Theorem}

\newtheorem{example}[theorem]{Example}
\newcommand{\bex}{\begin{example}}
\newcommand{\eex}{\end{example}}

\newtheorem{lemma}[theorem]{Lemma}

\newtheorem{proposition}[theorem]{Proposition}

\begin{document}

\advance\topmargin by 0.5in

\begin{center}

\vspace{.5truecm}

{\Large \bf \bf A New Convex Relaxation for Tensor Completion}\\
\vspace{.75truecm}

\end{center}


\begin{center}
{\bf Bernardino Romera-Paredes}\\
Department of Computer Science and UCL Interactive Centre\\
University College London\\
Gower Street, WC1EBT, London, UK\\
\texttt{bernardino.paredes.09@ucl.ac.uk} \\

\vspace{.42truecm}

{\bf Massimiliano Pontil} \\
Department of Computer Science and\\
Centre for Computational Statistics and Machine Learning \\
University College London\\
Gower Street, WC1EBT, London, UK\\
\texttt{m.pontil@cs.ucl.ac.uk} 
\end{center}

\vspace{.48truecm}

%

\newcommand{\fix}{\marginpar{FIX}}
\newcommand{\new}{\marginpar{NEW}}


\begin{abstract}
\noindent We study the problem of learning a tensor from a set of linear measurements. A prominent methodology for this problem is based on a generalization of trace norm regularization, which has been used extensively for learning low rank matrices, to the tensor setting. 
In this paper, we highlight some limitations of this approach and propose an alternative convex relaxation on the Euclidean ball. 
We then describe a technique to solve the associated regularization problem, which builds upon the alternating direction method of multipliers. Experiments on one synthetic dataset and two real datasets indicate that the proposed method improves significantly over tensor trace norm regularization in terms of estimation error, while remaining computationally tractable.
\end{abstract}
\vspace{.1truecm}
\section{Introduction}
During the recent years, there has been a growing interest on the problem of learning a tensor from a set of linear measurements, such as a subset of its entries, see \cite{Gandy,Liu,SignorettoA,Signoretto2011,Tomioka2,Tomioka3,Tomioka4} and references therein. This methodology, which is also referred to as tensor completion, has been applied to various fields, ranging from collaborative filtering \cite{Kara}, to computer vision \cite{Liu}, to medical imaging \cite{Gandy}, among others. In this paper, we propose a new method to tensor completion, which is based on a convex regularizer which encourages low rank tensors and develop an algorithm for solving the associated regularization problem.

Arguably the most widely used convex approach to tensor completion is based upon the extension of trace norm regularization \cite{Srebro} to that context. This involves computing the average of the trace norm of each matricization of the tensor \cite{Kolda2009}. 
A key insight behind using trace norm regularization for matrix completion is that this norm provides a tight convex relaxation of the rank of a matrix defined on the spectral unit ball \cite{Fazel}. Unfortunately, the extension of this methodology to the more general tensor setting presents some difficulties. In particular, we shall prove in this paper that the tensor trace norm is not a tight convex relaxation of the tensor rank.

The above negative result stems from the fact that the spectral norm, used to compute the convex relaxation for the trace norm, is not an invariant property of the matricization of a tensor. 
This observation leads us to take a different route and study afresh the convex relaxation of tensor rank on the Euclidean ball. We show that this relaxation is tighter than the tensor trace norm, and we describe a technique to solve the associated regularization problem. This method builds upon the alternating direction method of multipliers and a subgradient method to compute the proximity operator of the proposed regularizer. Furthermore, we present numerical experiments on one synthetic dataset and two real-life datasets, which indicate that the proposed method improves significantly over tensor trace norm regularization in terms of estimation error, while remaining computationally tractable.


The paper is organized in the following manner. In Section \ref{sec:pre}, we describe the tensor completion framework. In Section \ref{sec:alt}, we highlight some limitations of the tensor trace norm regularizer and present an alternative convex relaxation for the tensor rank. In Section \ref{sec:adm}, we describe a method to solve the associated regularization problem. In Section \ref{sec:exp}, we report on our numerical experience with the proposed method. Finally, in Section \ref{sec:conclusion}, we summarize the main contributions of this paper and discuss future directions  of research.


\section{Preliminaries}
\label{sec:pre}
In this section, we begin by introducing some notation and then proceed to describe the learning problem.
We denote by ${\mathbb N}$ the set of natural numbers and, for every $k \in {\mathbb N}$, we define $[k]=\{1,\dots,k\}$. Let $\NN \in \N$ and let\footnote{For simplicity we assume that $p_n \geq 2$ for every $n \in [N]$, otherwise we simply reduce the order of the tensor without loss of information.} $\II_1,\dots,\II_\NN \geq 2$.  An $\NN$-order tensor $\bW \in {\mathbb R}^{\II_1\times \cdots \times \II_\NN}$, is a collection of real numbers $(\W_{i_1,\dots,i_\NN}: i_n \in [\II_n], n \in [\NN])$. Boldface Euler scripts, e.g. $\bW$, will be used to denote tensors of order higher than two. Vectors are $1$-order tensors and will be denoted by lower case letters, e.g. $x$ or $a$; matrices are $2$-order tensors and will be denoted by upper case letters, e.g. $W$.  If $x \in \R^d$ then for every $r \leq s \leq d$, we define $x_{r:s} := (x_i: r \leq i \leq s)$. We also use the notation $p_{\min}=\min\{p_{1},\ldots,p_{N}\}$ and $p_{\max}=\max\{p_{1},\ldots,p_{N}\}$.  

A mode-$n$ fiber of a tensor $\bW$ is a vector composed of the elements of $\bW$ obtained by fixing
all indices but one, corresponding to the $n$-th mode. This notion is a higher order analogue of columns (mode-$1$ fibers) and rows (mode-$2$ fibers) for matrices. The mode-$n$ matricization (or unfolding) of $\bW$, denoted 
by $W_{\left(n\right)}$, is a matrix obtained by arranging the mode-$n$ fibers of $\bW$ so that each of them is a column of $W_{\left(n\right)}\in\mathbb{R}^{\II_{n}\times J_n}$, where $J_n:= \prod_{k\neq n} \II_{k}$. Note that the order of the columns is not important as long as it is consistent.

We are now ready to describe the learning problem. We choose a linear operator $\I: {\mathbb R}^{\II_1\times \cdots \times \II_\NN}  \rightarrow \R^m$, representing a set of linear measurements obtained from a target tensor $\X^0$ as $y = \I(\X^0) + \xi$, where $\xi$ is some disturbance noise. In this paper, we mainly focus on tensor completion, in which case the operator $\I$ measures elements of the tensor. That is, we have $\I(\X^0)=(\X^0_{i_1(j),\dots,i_N(j)}: j \in [m])$, where, for every $j \in [m]$ and $n \in [N]$, the index $i_n(j)$ is a prescribed 
integer in the set $[p_n]$.
Our aim is to recover the tensor $\X^0$ from the data $(\I,y)$. To this end, we solve the regularization problem
\beq
\label{eq:www}
\min \big\{\|y - \I(\X)\|^2_2 + \gamma R(\X): {\X \in {\mathbb R}^{\II_1\times \cdots \times \II_\NN}} \big\}
\eeq
where $\gamma$ is a positive parameter which may be chosen by cross validation. The role of the regularizer $R$ is to encourage tensors $\X$ which have a simple structure in the sense that they involve a small number of ``degrees of freedom''. 
A natural choice is to consider the average of the rank of the tensor's matricizations. Specifically, we consider the combinatorial regularizer
\beq
R(\X) = 
\frac{1}{N} 
\sum\limits_{n=1}^N {\rm rank}(W_{(n)}). 
\label{eq:rank}
\eeq
Finding a convex relaxation of this regularizer has been the subject of recent works \cite{Gandy,Liu,Signoretto2011}.  They all agree to use the trace norm for tensors as a convex proxy of $R$. This is defined as the average of the trace norm of each matricization of $\bW$, that is,
\beq
\|\bW \|_{\rm tr} = 
\frac{1}{N}
\sum\limits_{n=1}^N \|W_{\left(n\right)}\|_{\rm tr}
\label{eq:compTN}
\eeq
where $\|W_{(n)}\|_{\rm tr}$ is the trace (or nuclear) norm of matrix $W_{(n)}$, namely the $\ell_1$-norm of the vector of singular values of matrix $W_{(n)}$ (see, e.g. \cite{HJ}).  Note that in the particular case of $2$-order tensors, functions \eqref{eq:rank} and \eqref{eq:compTN} coincide with the usual notion of rank and trace norm of a matrix, respectively. 

A rational behind the regularizer \eqref{eq:compTN} is that 
the trace norm is the tightest convex lower bound to the rank of a matrix on the spectral unit ball, see \cite[Thm. 1]{Fazel}. This lower bound is given by the convex envelope of the function
\begin{equation}
 \Psi(W) = \left\{
    \begin{array}{ll}
      {\rm rank}(W), & \text{if } \|W\|_{\infty} \leq 1\\
      +\infty, & \text{otherwise} \end{array}
\right.
\label{eq:lollo}
\end{equation}
where $\|\cdot\|_\infty$ is the spectral norm, namely the largest singular value of $W$. The convex envelope can be derived by computing the double conjugate of $\Psi$. This is defined as
\beq
\label{eq:comb}
\Psi^{**}(W) = \sup \left\{ \lb W,S\rb - \Psi^*(W): S \in \R^{p_1\times p_2} \right\}
\eeq
where $\Psi^*$ is the conjugate of $\Psi$, namely $\Psi^*(S) = \sup\left\{ \lb W,S\rb - \Psi(W): W \in \R^{p_1\times p_2}\right\}$.

Note that $\Psi$ is a spectral function, that is, $\Psi(W) = \psi(\sigma(W))$ where $\psi: \R_+^d \rightarrow \R$ denotes the associated symmetric gauge function.
Using von Neumann's trace theorem (see e.g. \cite{HJ}) it is easily seen that $\Psi^*(S)$ is also a spectral function. That is, $\Psi^*(S) = \psi^*(\sigma(S))$, where
$$
\psi^*(\sigma) = \sup\left\{ \lb \sigma,w\rb - \psi(w): w \in \R_+^{d}\right\},~~~{\rm with~}d := \min(p_1,p_2).
$$
We refer to \cite{Fazel} for a detailed discussion of these ideas. We will use this equivalence between spectral and gauge functions repeatedly in the paper.
%

\section{Alternative Convex Relaxation}
\label{sec:alt}

In this section, we show that the tensor trace norm is not a tight convex relaxation of the tensor rank $R$ in equation \eqref{eq:rank}. We then propose an alternative convex relaxation for this function.

Note that due to the composite nature of the function $R$, computing its convex envelope is a challenging task  
and one needs to resort to approximations. In \cite{SignorettoA}, the authors note that the tensor trace norm $\|\cdot\|_{\rm tr}$ in equation \eqref{eq:compTN}  is a convex lower bound to $R$ on the set 
$$
\mathcal{G}_{\Sp}:=\left\{ \boldsymbol{{\bW}} 
\in\mathbb{R}^{p_{1}\times
\cdots \times p_{N}}\,:\,\left\Vert W_{\left(n\right)}\right\Vert _{\Sp}\leq1,~\forall n\in[N]\right\}.
$$
The key insight behind this observation is summarized in Lemma \ref{lem:1}, which we report in Appendix \ref{sec:useful_Lemma}. 
However, the authors of \cite{SignorettoA} leave open the question of whether the tensor trace norm is the convex
envelope of $R$ on the set $\mathcal{G}_{\Sp}$. In the following, we will prove that this question has a negative answer
by showing that there exists a convex function $\Omega \neq \|\cdot\|_{\rm tr}$ which underestimates the function $R$ on 
$\mathcal{G}_{\Sp}$ and such that for some tensor $\bW \in \mathcal{G}_{\Sp}$ it holds that $ \Omega(\bW) > \|\bW\|_{\rm tr}$.

To describe our observation we introduce the set
$$
\mathcal{G}_{\Fr}:=\left\{ \bW \in \mathbb{R}^{p_{1}\times \ldots\times p_{N}} : \|\bW\|_2 \leq 1
\right\}
$$ 
where $\|\cdot\|_2$ is the Euclidean norm for tensors, that is, 
$$
\|\bW\|_2^2 := \sum_{i_1=1}^{p_1} \cdots \sum_{i_N=1}^{p_N} (\W_{i_1,\dots,i_N})^2.
$$ 
We will choose 
\beq
\Omega(\bW) = \Omega_\alpha(\bW):=
\frac{1}{N} \sum_{n=1}^N 	\omega^{**}_\alpha\left(\sigma\left(W_{(n)}\right)\right)
\label{eq:our}
\eeq
where $\omega_\alpha^{**}$ is the convex envelope of the cardinality of a vector on the $\ell_{2}$-ball
of radius $\alpha$ and we will choose $\alpha=\sqrt{p_{{\rm min}}}$. Note, by Lemma \ref{lem:1} stated in Appendix \ref{sec:useful_Lemma}, that, for every $\alpha > 0$, function  
$\Omega_\alpha$ is a convex lower bound of function $R$ on the set $\alpha \hspace{.03truecm}\mathcal{G}_{\Fr}$. 


Below, for every vector $s \in \R^d$ we denote 
by $s^\downarrow$ the vector obtained by reordering the components of $s$ so that they are non increasing in absolute value, that is,  
$|s^\downarrow_1| \geq \cdots \geq |s^\downarrow_d|$.
\begin{lemma} Let $\omega^{**}_\alpha$ be the convex envelope
of the cardinality function 
on the $\ell_{2}$-ball of radius $\alpha$. Then, for every $x\in \R^d$ such that $\|x\|_2 = \alpha$, it holds that $\omega_\alpha^{**}\left(x\right)={\rm card}\left(x\right)$.
\label{lem:bern}
\end{lemma}
\begin{proof}
First, we note that the conjugate of the function ${\rm card}$
on the $\ell_2$ ball of radius $\alpha$ is given by the formula
\beq
\omega_\alpha^{*}\left(s\right)=\underset{\|y\|_2 \leq \alpha}{{\rm sup}}\{\lb s,y\rb-{\rm card}\left(y\right)\}=\underset{r\in\left\{ 0,\ldots,d\right\} }\max\{\alpha\| s^\downarrow_{1:r}\|_{2}-r\}.
\label{eq:fstar}
\eeq
Hence, by the definition of the double conjugate, we have, for every $s \in \R^d$ that
$$
\omega_\alpha^{**}\left(x\right)\geq \lb s,x\rb - \underset{r\in\left\{ 0,\ldots,d\right\} }{{\rm max}}\{\alpha\|s^\downarrow_{1:r}\| _{2}-r\}.
$$
In particular, if $s=kx$ for some $k>0$ this inequality becomes
$$
\omega_\alpha^{**}(x)\geq k\| x\|_{2}^{2} - \underset{r\in\left\{ 0,\ldots,d\right\} }\max(\alpha k\| x^\downarrow_{1:r}\|_{2}-r). 
$$
If $k$ is large enough, the maximum is attained at  
$r={\rm card}(x)$. Consequently, 
$$
\omega_\alpha^{**}(x)\geq k\alpha^{2}-k\alpha^{2}+{\rm card}(x)={\rm card}(x).
$$
By the definition of the convex envelope, it also holds that $\omega_\alpha^{**}(x) \leq {\rm card}(x)$. The result follows.
\end{proof}

The next lemma provides, together with Lemma \ref{lem:bern}, a sufficient condition for the existence of a tensor $\bW \in {\cal G}_\Sp$ at which the proposed regularizer is strictly larger than the tensor trace norm. 

\begin{lemma} If $N \geq 3$ and $p_1,\dots,p_N$ are not all equal to each other, then there exists $\bW \in  \R^{p_1 \times \cdots \times p_N}$ such that:~{(\em a)} $\|\bW\|_2 = \sqrt{p_{\min}}$,~{\em (b)} $\bW \in {\cal G}_\Sp$,~{\em (c)} $\min\limits_{n \in [N]} \rank(W_{(n)}) < \max\limits_{n \in [N]} \rank(W_{(n)})$.
\label{lem:222}
\end{lemma}
\begin{proof}
Without loss of generality we assume that $p_1 \leq \cdots \leq p_N$. By hypothesis $p_1 < p_N$. 
First we consider the special case
\begin{equation}
p_1=\dots=p_{N-1},~{\rm and}\hspace{.05truecm}~ p_{N} = p_1+1.
\label{eq:easy}
\end{equation}
We define a class of tensors $\bW$ by choosing a singular value decomposition for their mode-$N$ matricization, 
\beq
\label{eq:qqq}
\W_{i_1,i_2,\dots,i_N} = \sum_{k=1}^{p_N} \sigma_{k} u^{k}_{i_N} v^{k}_{i_1,\dots,i_{N-1}}
\eeq
where $\sigma_1 =  \cdots = \sigma_{p_N} = \sqrt{p_1/(p_1+1)}$, the vectors $u^{k} \in \R^{p_N}, \forall k \in [p_N]$ are orthonormal and the vectors $v^{k}\in \R^{p_1 p_2\cdots p_{N-1} }, \forall k \in [p_N] $ are orthonormal as well. Moreover, we choose $v^{k}$ as 
\begin{equation}
  v^{k}_{i_1,\dots,i_{N-1}} = \left\{
    \begin{array}{cll}
      1 & ~\text{if } i_1= \cdots = i_{N-1} = k, & ~k <p_N\\
      \frac{1}{\sqrt{p_1}} &~ \text{if } i_2=\cdots=i_{N-1} ={\rm module}(i_1,p_1)+1,& ~k=p_N\\
      0 & ~\text{otherwise}. &
    \end{array} \right.
    \label{eq:choice}
\end{equation}
By construction the matrix $W_{(N)}$ has rank equal to $p_N$ and Frobenius norm equal to $\sqrt{p_1}$. Thus properties (a) and (c) hold true. It remains to show that $\bW$ satisfies property (b). To this end, we will show, for every $n\in [N]$ and every $x\in \R^{p_n}$, that 
$$
\|W_{(n)}\trans x\|_2 \leq \|x\|_2.
$$
The case $n=N$ is immediate. If $n=1$ we have
\begin{eqnarray}
\nonumber
\|W_{(1)}\trans x\|_2^2 & = & \sum_{i_2,\dots,i_N} \left(\sum_{k} \sigma_k \sum_{i_1} u^{k}_{i_N} v^{k}_{i_1,\dots,i_{N-1}} x_{i_1}
 \right)^2 \\ \nonumber  ~ & = &  \sum_{i_2,\dots,i_N} \sum_{k,\ell} \sum_{i_1,j_1} x_{i_1}  x_{j_1} \sigma_k \sigma_\ell u^{k}_{i_N}  u^{\ell}_{i_N}v^{k}_{i_1,i_2,\dots,i_{N-1}}  v^{\ell}_{j_1,i_2,\dots,i_{N-1}} 
\\ \nonumber
 ~ & = &  \sum_{k} \sigma_k^2 \sum_{i_1,j_1} x_{i_1}  x_{j_1}  \left(\sum_{i_2,\dots,i_{N-1}} v^{k}_{i_1,i_2,\dots,i_{N-1}} v^{k}_{j_1,i_2,\dots,i_{N-1}} \right)
 \nonumber \\
 ~ & = &   \sum_{k} \sigma_k^2 x_k^2 + \frac{\sigma_{p_N}^2}{p_1}\sum_{k}x_k^2 
  = \| x\|_2^2
 \nonumber
\end{eqnarray}
where we used $\sum_{i_N} u^{k}_{i_N} u^\ell_{i_N} = \delta_{k,\ell}$ in the third equality, equation \eqref{eq:choice} and a direct computation in the fourth equality, and the definition of $\sigma_k$ in the last equality. \\
All other cases, namely $n =2,\dots,N-1$, are conceptually identical, so we only discuss the case $n=2$. We have
\begin{eqnarray}
\nonumber
\|W_{(2)}\trans x\|_2^2 & = & \sum_{i_1,i_3,\dots,i_N} \left(\sum_{k} \sigma_k \sum_{i_2} u^{k}_{i_N} v^{k}_{i_2,\dots,i_{N-1}} x_{i_2}
 \right)^2 \\ \nonumber  ~ & = &  \sum_{i_1,i_3,\dots,i_N} \sum_{k,\ell} \sum_{i_2,j_2} x_{i_2}  x_{j_2} \sigma_k \sigma_\ell u^{k}_{i_N}  u^{\ell}_{i_N}v^{k}_{i_1,i_2,\dots,i_{N-1}}  v^{\ell}_{i_1,j_2,\dots,i_{N-1}} 
\\ \nonumber
 ~ & = & \sum_{k} \sigma_k^2  \sum_{i_2,j_2} \left(x_{i_2}  x_{j_2} \sum_{i_1,i_3,\dots,i_{N=1}}v^{k}_{i_1,i_2,\dots,i_{N-1}} v^{k}_{i_1,j_2,\dots,i_{N-1}} \right)
 \nonumber \\
 ~ & = &  \sum_{k} \sigma_k^2 x_k^2 + \frac{\sigma_{p_N}^2}{p_1}\sum_{k}x_k^2 
  = \| x\|_2^2
 \nonumber
\end{eqnarray}
where again we used $\sum_{i_N} u^{k}_{i_N} u^\ell_{i_N} = \delta_{k,\ell}$ in the third equality, equation \eqref{eq:choice} and a direct computation in the fourth equality, and the definition of $\sigma_k$ in the last equality. \\
Finally, if assumption \eqref{eq:easy} is not true we set $\W_{i_1,\dots,i_N} = 0$ if $i_n \geq p_1+1$, for some $n \leq N-1$ or $i_N > p_1+1$. We then proceed as in the case $p_1=\dots=p_{N-1}$ and $p_{N} = p_1+1$. 
\end{proof}

We are now ready to present the main result of this section.
\begin{proposition}
Let $p_1,\dots,p_N \in \N$, let $\|\cdot\|_{\rm tr}$ be the tensor trace norm in equation \eqref{eq:compTN} and let $\Omega_\alpha$ be the function in equation \eqref{eq:our} for $\alpha = \sqrt{p_{\min}}$. If $p_{\min} < p_{\max}$, then there are infinitely many tensors $\bW \in \calG_{\Sp}$ such that $\Omega_\alpha(\bW) > \|\bW\|_{\rm tr}$. Moreover, for every $\bW \in \calG_{\Fr}$, it holds that $\Omega_1(\bW) \geq \|\bW\|_{\rm tr}$.
\end{proposition}
\begin{proof}
By construction $\Omega_\alpha(\bW)\leq R(\bW)$ for every $\bW \in \alpha \calG_2$. Since $\calG_{\Sp} \subset \alpha \mathcal{G}_{\Fr}$ then $\Omega_\alpha$ is a convex lower bound for the tensor rank $R$ on the set $\mathcal{G}_{\Sp}$ as well. The first claim now follows by Lemmas \ref{lem:bern} and \ref{lem:222}. Indeed, all tensors obtained following the process described in Lemma \ref{lem:222} have the property that 
\begin{eqnarray}
\nonumber
\|\bW \|_{\rm tr} &= & \frac{1}{N} \sum_{n=1}^N \|\sigma(W_{(n)})\|_1 \\ \nonumber 
& = & \frac{1}{N}\left(p_{{\rm min}}(N-1)+\sqrt{p_{\min}^{2}+p_{\min}}\right) \\ \nonumber
& < & \frac{1}{N}\left(p_{\min}(N-1)+p_{\min}+1\right) =\Omega(\bW)=R(\bW).
\end{eqnarray}
 
Furthermore there are infinitely many such tensors which satisfy this claim since the left singular vectors can be arbitrarily chosen in equation \eqref{eq:qqq}. \\To prove the second claim, we note that since $\omega_1^{**}$ is the convex envelope of the cardinality ${\rm card}$ on the Euclidean unit ball, then it holds that $\omega_1^{**}(\sigma) \geq \|\sigma\|_1$ for every vector $\sigma$ such that $\|\sigma\|_2 \leq 1$.  Consequently,
$$
\Omega_1(\bW) = \frac{1}{N} \sum_{n=1}^N \omega_1^{**}\left(\sigma\left(W_{(n)}\right)\right) \geq \frac{1}{N}  \sum_{n=1}^N \|\sigma(W_{(n)})\|_1 = \|\bW\|_{\rm tr}.$$
\end{proof}

The above 
result stems from the fact that the spectral norm is not an invariant property of the matricization of a tensor, whereas the Euclidean (Frobenius) norm is. This observation leads us to further study the function $\Omega_\alpha$.  

\section{Optimization Method}
\label{sec:adm}
In this section, we explain how to solve the regularization problem associated with the proposed regularizer \eqref{eq:our}. For this purpose, we first recall the alternating direction
method of multipliers (ADMM) \cite{Bertsekas}, which
was conveniently applied to tensor trace norm regularization in \cite{Gandy,SignorettoA}.

\subsection{Alternating Direction Method of Multipliers (ADMM)}

To explain ADMM we consider a more general problem comprising both tensor trace norm regularization and the regularizer we propose,
\begin{equation}
\underset{\boldsymbol{{\cal W}}}{{\rm min}} \left\{ E\left(\boldsymbol{{\cal W}}\right)+\gamma\underset{n=1}{\overset{N}{\sum}}\Psi\left(W_{(n)}\right)\right\}\label{eq:General}
\end{equation}
where $E(\bW)$ is an error term such as $\|y - \I(\bW)\|^2_2$ and $\Psi$ is a convex spectral function. It is defined, for every matrix $A$, as
$$
\Psi(A)=\psi(\sigma(A))
$$
where $\psi$ is a gauge function, namely a function which is symmetric and invariant under permutations. In particular, if $\psi$ is the $\ell_1$ norm then problem \eqref{eq:General} corresponds to tensor trace norm regularization, whereas if $\psi= \omega^{**}_\alpha$ it implements the proposed regularizer.

Problem \eqref{eq:General} poses some difficulties because the terms under the summation are interdependent, 
that is, the different matricizations of $\bW$ have
the same elements rearranged in a different way. In order to overcome
this difficulty, the authors of \cite{Gandy,SignorettoA} proposed to use ADMM as a natural way to decouple the regularization term appearing in problem \eqref{eq:General}. This strategy is based on the introduction
of $N$ auxiliary tensors, $\boldsymbol{{\cal B}}_{1},\ldots,\boldsymbol{{\cal B}}_{N}\in\mathbb{R}^{p_1\times \cdots \times p_N}$, so that problem (\ref{eq:General}) can be reformulated as\footnote{The somewhat cumbersome notation 
$B_{n(n)}$ denotes the mode-$n$ matricization of tensor $\bB_n$, that is, $B_{n(n)} = (\boldsymbol{{\cal B}}_n)_{(n)}$.}
\begin{equation}
\underset{\mathbf{\boldsymbol{{\cal W}}},\boldsymbol{{\cal B}}_{1},\ldots,\mathbf{\boldsymbol{{\cal B}}}_{N}}{{\rm min}}\left\{\frac{1}{\gamma}E\left(\boldsymbol{{\cal W}}\right)+\underset{n=1}{\overset{N}{\sum}}\Psi\left(B_{n(n)}\right)\,\,:~\mathbf{\boldsymbol{{\cal B}}}_{n}=\boldsymbol{{\cal W}},\, n\in\left[N\right]\right\}\label{eq:Main2}
\end{equation}
The corresponding augmented Lagrangian (see e.g. \cite{Bertsekas,Boyd}) is given by 
\begin{equation}
\mathcal{L}\left(\boldsymbol{{\cal W}},\boldsymbol{{\cal B}},\boldsymbol{{\cal A}}\right)=\frac{1}{\gamma}E\left(\boldsymbol{{\cal W}}\right)+\underset{n=1}{\overset{N}{\sum}}\left(\Psi\left(B_{n(n)}\right)-\left\langle \boldsymbol{{\cal A}}_{n},\boldsymbol{{\cal W}}-\mathbf{\boldsymbol{{\cal B}}}_{n}\right\rangle +\frac{\beta}{2}\left\Vert \boldsymbol{{\cal W}}-\boldsymbol{{\cal B}}_{n}\right\Vert _{2}^{2}\right),
\end{equation}
where $\lb \cdot,\cdot\rb$ denotes the scalar product between tensors, $\beta$ is a positive parameter and $\boldsymbol{{\cal A}}_{1},\ldots\boldsymbol{{\cal A}}_{N} \in \R^{p_1\times \cdots \times p_N}$ are the set of Lagrange multipliers associated with the constraints in problem (\ref{eq:Main2}).

ADMM is based on the following iterative scheme
\begin{eqnarray}
\boldsymbol{{\cal W}}^{\left[i+1\right]} & \leftarrow &\underset{\boldsymbol{{\cal W}}}{{\rm argmin}}~\mathcal{L}\left(\boldsymbol{{\cal W}},\boldsymbol{{\cal B}}^{[i]},\boldsymbol{{\cal A}}^{[i]}\right)\label{eq:ADMM1} \\
\boldsymbol{{\cal B}}_{n}^{\left[i+1\right]} & \leftarrow &\underset{\boldsymbol{{\cal B}}_{n}}{{\rm argmin}}~\mathcal{L}\left(\boldsymbol{{\cal W}}^{\left[i+1\right]},\boldsymbol{{\cal B}},\boldsymbol{{\cal A}}^{[i]}\right)\label{eq:ADMM2} \\
\boldsymbol{{\cal A}}_{n}^{\left[i+1\right]} & \leftarrow & \boldsymbol{{\cal A}}_{n}^{\left[i\right]}-\left(\beta\boldsymbol{{\cal W}}^{\left[i+1\right]}-\boldsymbol{{\cal B}}_{n}^{\left[i+1\right]}\right).\label{eq:ADMM3}
\end{eqnarray}

Step (\ref{eq:ADMM3}) is straightforward, whereas step (\ref{eq:ADMM1}) is described in \cite{Gandy}. 
Here we focus on the step (\ref{eq:ADMM2}) since this is the only problem which involves function
$\Psi$. We restate it with more explanatory notations as
$$
\underset{B_{n(n)}}{\rm argmin}
\left\{
\Psi\left(B_{n(n)}\right)-\left\langle A_{n(n)},W_{(n)}-B_{n(n)}\right\rangle 
+\frac{\beta}{2}\left\Vert W_{(n)}-B_{n(n)}\right\Vert _{2}^{2}
\right\}.
$$
By completing the square in the right hand side, the solution of this problem is given by
$$
\hat{B}_{n(n)}={\rm prox}_{\frac{1}{\beta}\Psi}
\left(X\right):= \underset{B_{n(n)}}{{\rm argmin}}
\left\{\frac{1}{\beta}\Psi\left(B_{n(n)}\right)+\frac{1}{2}\left\Vert B_{n(n)}-X\right\Vert _{2}^{2}\right\}
$$
where $X=W_{\left(n\right)}-\frac{1}{\beta}A_{n\left(n\right)}$.
By using properties of proximity operators (see e.g. \cite[Prop. 3.1]{andy-unpub}) we know that if
$\psi$ is a gauge function then
$$
{\rm prox}_{\frac{1}{\beta}\Psi}\left(X\right)=U_X {\rm diag\left({\rm prox}_{\frac{1}{\beta}\psi}\left(\sigma(X)\right)\right)}V_{X}^{\top}
$$
where $U_{X}$ and $V_{X}$ are the orthogonal matrices formed by the left and right singular vectors of $X$, respectively.

If we choose $\psi=\left\Vert \cdot\right\Vert _{1}$ the associated proximity operator is the well-known soft thresholding operator, that is, ${\rm prox}_{\frac{1}{\beta}\left\Vert \cdot\right\Vert _{1}}\left(\sigma\right)=v$, where the vector $v$ has components
$$
v_{i}={\rm sign}\left(\sigma_{i}\right)\left(\left|\sigma_{i}\right|-\frac{1}{\beta}\right).
$$
On the other hand, if we choose $\psi=\omega^{**}_\alpha$, we need to compute ${\rm prox}_{\frac{1}{\beta}\omega^{**}_\alpha}$.
In the next section, we describe a method to accomplish this task.

\subsection{Computation of the Proximity Operator}
In order to compute the proximity operator of the function ${\frac{1}{\beta}\omega_\alpha^{**}}$ we will use several properties of proximity calculus.
First, we use the formula (see e.g. \cite{Combettes}) ${\rm prox}_{g^{*}}\left(x\right)=x-{\rm prox}_{g}\left(x\right)$ for $g^{*}=\frac{1}{\beta}\omega_\alpha^{**}$. Next we use a property of conjugate functions from \cite{Shor,Urruty}, which states that $g(\cdot)=\frac{1}{\beta}\omega_\alpha^{*}(\beta \cdot)$. Finally, by the scaling property of proximity operators \cite{Combettes}, we have that ${\rm prox}_{g}\left(x\right)=\frac{1}{\beta}{\rm prox}_{\beta \omega_\alpha^{*}}\left(\beta x\right)$.

It remains to compute the proximity operator of a multiple of the function $\omega_\alpha^*$ in equation \eqref{eq:fstar}, that is, 
for any $\beta>0$, $y\in\mathcal{S}$, we wish to compute
$$
{\rm prox}_{\beta \omega_\alpha^{*}}\left(y\right)=\underset{w}{{\rm argmin}}\left\{ h\left(w\right) : w\in\mathcal{S}\right\} 
$$
where we have defined $\mathcal{S} := \{w \in \R^d: w_1 \geq \cdots \geq w_d \geq 0\}$ and 
$$
h\left(w\right)=\frac{1}{2}\left\Vert w-y\right\Vert _{2}^{2}+\beta~\underset{r=0}{\overset{d}{{\rm max}}}\left\{\alpha\left\Vert w_{1:r}\right\Vert _{2}-r\right\}.
$$
In order to solve this problem we employ the projected subgradient
method, see e.g. \cite{Subgradient}. It consists in applying two steps at each
iteration. First, it advances along a negative subgradient of the current solution; second, it projects the resultant point onto the feasible set $\mathcal{S}$. In fact, according to \cite{Subgradient}, it is sufficient to compute an approximate projection, a step which we describe in Appendix \ref{sec:projection}. To compute a subgradient of $h$ at $w$, we first
find any integer $\tilder$ such that ${\tilder}\in \underset{r=0}{\overset{d}{{\rm argmax}}}\left\{\alpha\left\Vert w_{1:r}\right\Vert _{2}-r\right\}$.
Then, we calculate a subgradient $\subgrad$ of the function $h$ at $w$ by the formula

\begin{equation*}
\subgrad_i = \left\{
    \begin{array}{ll}
      \left(1+\frac{\alpha \beta}{\left\Vert w_{1:{\tilder}}\right\Vert _{2}}\right)w_i-y_i, & \text{if } i \leq \tilder,\\
       w_i - y_i,  & \text{otherwise.}
\end{array} \right.
\end{equation*}
Now we have all the ingredients to apply the projected subgradient method, which is summarized in Algorithm \ref{alg:1}. 
In our implementation we stop the algorithm when an update of ${\hat w}$ is not made for more than $10^{3}$ iterations.

\begin{algorithm}[t]
\begin{algorithmic}
\STATE \textbf{Input}: $y\in \mathbb{R}^d$, $\alpha,\beta>0$.
\STATE \textbf{Output}: $\hat{w}\in\mathbb{R}^d$.
\STATE \textbf{Initialization}: initial step $\tau_{0}=\frac{1}{2}$, initial and best found solution $w^0=\hat{w}=P_{{\cal S}}(y) \in \R^d$.
\FOR {$t=1,2,\dots$}
\STATE $\tau\leftarrow\frac{\tau_{0}}{\sqrt{t}}$
\STATE Find $\tilder$ such that $\tilder\in {\rm argmax}\left\{\alpha\| w^{t-1}_{1:r}\|_{2}-r: 0 \leq r \leq d\right\}$
\STATE $\tilde{w}_{1:\tilder} \leftarrow w^{t-1}_{1:\tilder}-\tau \left(w^{t-1}_{1:\tilder}\left(1+\frac{\alpha\beta}{\left\Vert w^{t-1}_{1:\tilder}\right\Vert _{2}}\right)-y_{1:\tilder}\right)$
\STATE $\tilde{w}_{\tilder+1:d}\leftarrow w^{t-1}_{\tilder+1:d}-\tau\left(w^{t-1}_{\tilder+1:d}-y_{\tilder+1:d}\right)$
\STATE $w_t\leftarrow \aP_{\mathcal{S}}\left(\tilde{w}\right)$
\STATE If~ $h(w_t)<h(\hat{w})$~ then $\hat{w}\leftarrow w_t$
\STATE If~ ``Stopping Condition = True''~ then terminate.

\ENDFOR
\end{algorithmic}
\caption{Computation of ${\rm prox}_{\beta \omega_\alpha^{*}}(y)$}
\label{alg:1}
\end{algorithm}

\section{Experiments}
\label{sec:exp}
We have conducted a set of experiments to assess whether there is
any advantage of using the proposed regularizer over the tensor trace norm for
tensor completion. First, we have designed a synthetic experiment
to evaluate the performance of both approaches under controlled conditions.
Then, we have tried both methods on two tensor completion real data problems.
In all cases, we have used a validation procedure to tune the hyper-parameter
$\gamma$, present in both approaches, among the values $\left\{ 10^{j}\,:\, j= -7,-6,\ldots,0 \right\}$.
In our proposed approach there is one further hyper-parameter, $\alpha$,
to be specified. It should take the value of the Frobenius norm of
any matricization of the underlying tensor. Since this is unknown, we propose to use the estimate

$$\hat{\alpha}=\sqrt{\left\Vert w\right\Vert _{2}^{2}+\left({\rm mean}(w)^{2}+{\rm var}(w)\right)\left(\underset{i=1}{\overset{N}{\prod}}p_{i}-m\right)}\,,$$

where $m$ if the number of known entries and $w\in\mathbb{R}^{m}$
contains their values. This estimator assumes that each value in $w$ is sampled from $\mathcal{N}({\rm mean}(w),{\rm var}(w))$,
where ${\rm mean}(w)$ and ${\rm var}(w)$ are the average and the variance
of the elements in $w$.

\subsection{Synthetic Dataset}
We have generated a $3$-order tensor $\bW^0\in\mathbb{R}^{40\times20\times10}$ by the following procedure. 
First we generated a tensor $\bW$ with ranks $\left(12,6,3\right)$ using Tucker decomposition (see e.g. \cite{Kolda2009})
$${\W}_{i_{1},i_{2},i_{3}}=
\sum_{j_1=1}^{12} \sum_{j_2=1}^6\sum_{j_3=1}^3
{\cal C}_{j_{1},j_{2},j_{3}}M_{i_{1},j_{1}}^{(1)}M_{i_{2},j_{2}}^{(2)}M_{i_{3},j_{3}}^{(3)},~~~(i_{1},i_{2},i_{3}) \in [40]\times[20]\times[10]
$$
where each entry of the Tucker decomposition components is sampled from
the standard Gaussian distribution $\mathcal{N}(0,1)$. We then created 
the ground truth tensor $\bW^0$ by the equation 
$$
\W^0_{i_1,i_2,i_3} = \frac{{\W}_{i_1,i_2,i_3} - {\rm mean}(\bW)}{{\rm std}(\bW)}+ \xi_{i_1,i_2,i_3}
$$
where ${\rm mean}(\bW)$ and ${\rm std}(\bW)$ are the mean and standard deviation of the elements of ${\bW}$ and the $\xi_{i_1,i_2,i_3}$ are i.i.d. Gaussian random variables with zero mean and variance $\sigma^2$.
We have randomly sampled $10\%$ of the elements of the tensor to compose the training
set, $45\%$ for the validation set, and the remaining $45\%$ for
the test set. After repeating this process $20$ times, we report the
average results in Figure \ref{fig:SynthResults} (Left). Having conducted a paired $t$-test for each value of $\sigma^{2}$,
we conclude that the visible differences in the performances are highly
significant, obtaining always $p$-values less than $0.01$ for $\sigma^{2}\leq10^{-2}$.

\begin{figure}
\begin{centering}
\includegraphics[width=0.49\columnwidth]{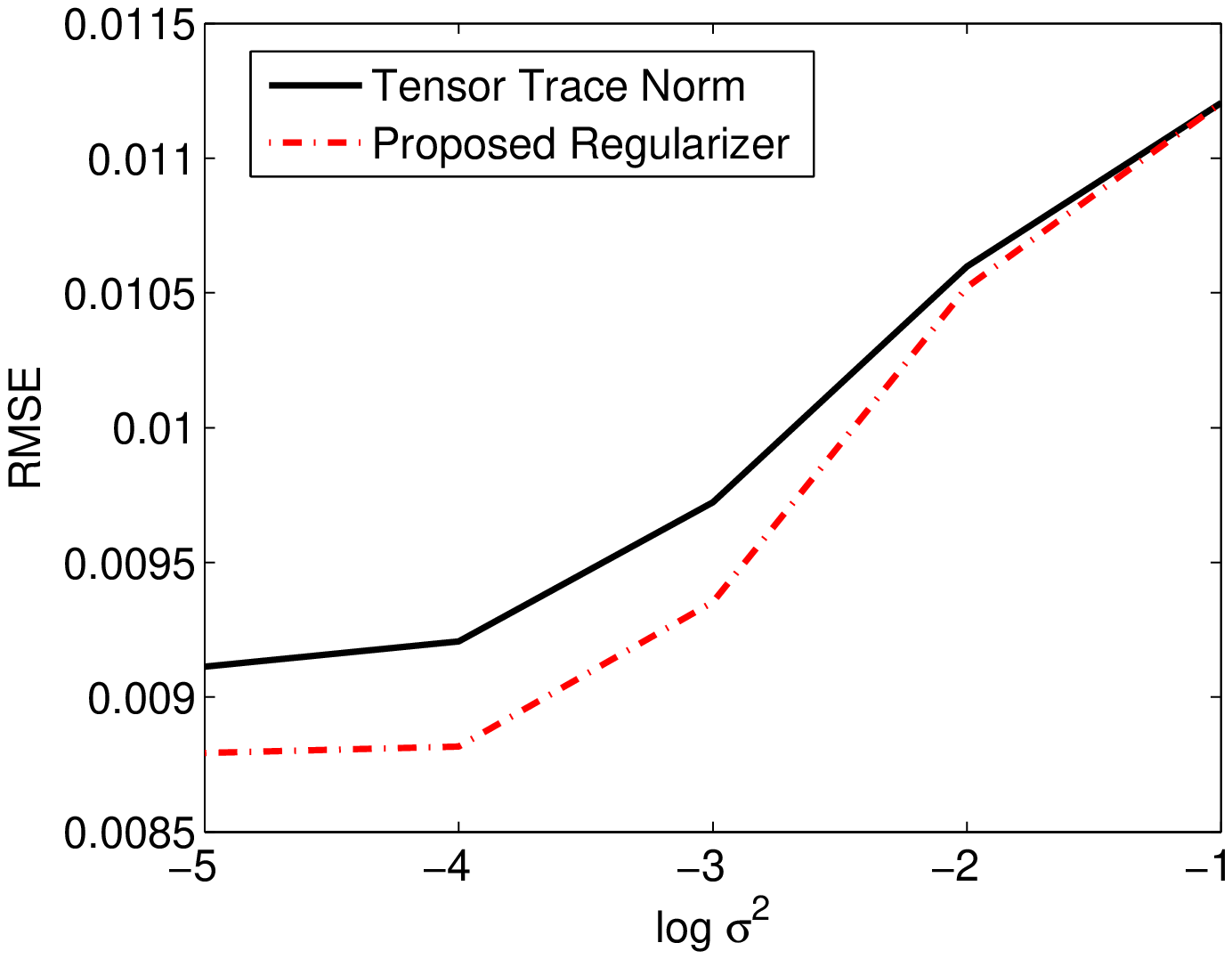}
\includegraphics[width=0.49\columnwidth]{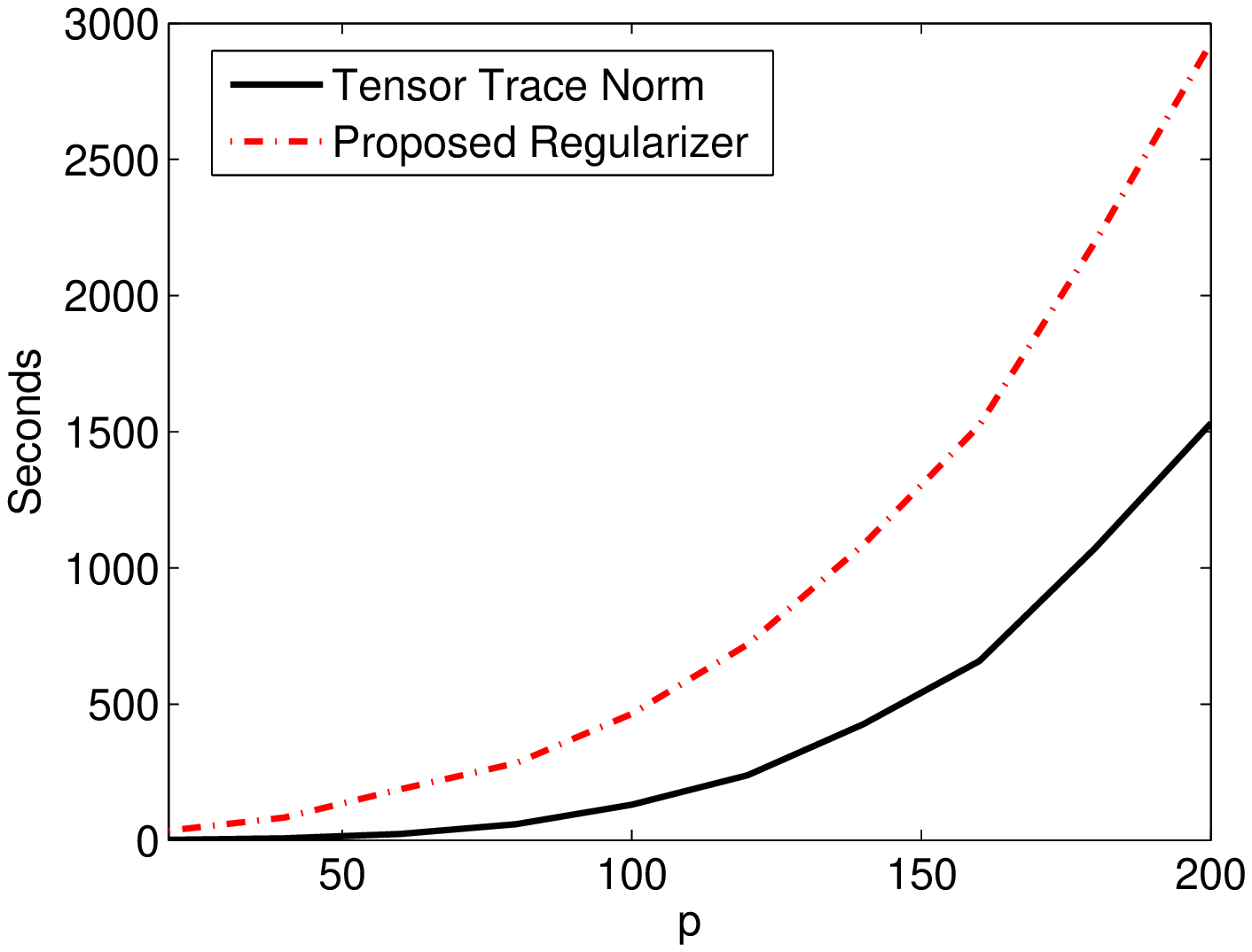}
\par\end{centering}

\caption{\label{fig:SynthResults} Synthetic dataset: (Left) Root Mean Squared Error
(RMSE) of tensor trace norm and the proposed regularizer. (Right)  Running time execution
of both algorithms for different sizes of the tensor.}
\end{figure}

Furthermore, we have conducted an experiment to test the running time
of both approaches. 
We have generated tensors $\boldsymbol{{\cal W}}^0\in\mathbb{R}^{p\times p\times p}$
for different values of $p \in \{ 20,\,40,\,\ldots,200\}$, following the same procedure as outlined above. The results are reported in
Figure \ref{fig:SynthResults} (Right). For low values of $p$, the ratio between the running time of our approach
and that of trace norm regularization is quite high. For example in the lowest value
tried for $p$ in this experiment, $p=20$, this ratio is $22.661$.
However, as the volume of the tensor increases, the ratio quickly decreases. For example, for $p=200$, the running time ratio is
$1.9113$. These outcomes are expected since when $p$ is low, the
most demanding routine in our method is the one described in Algorithm
\ref{alg:1}, where each iteration is of order $O\left(p\right)$
and $O\left(p^{2}\right)$ in the best and worst case, respectively.
However, as $p$ increases the singular value decomposition routine,
which is common to both methods, becomes the most demanding because
it has a time complexity $O\left(p^{3}\right)$ \cite{Golub}. Therefore,
we can conclude that even though our approach is slower than the trace
norm based method, this difference becomes much smaller as the size of the tensor increases.

\subsection{School Dataset}
The first real dataset we have tried is the Inner London Education
Authority (ILEA) dataset\footnote{Available at http://www.bristol.ac.uk/cmm/learning/support/datasets/ilea567.zip.}
. 
It is composed of examination marks ranging from $0$ to $70$,
of $15362$ students which are described by a set of attributes such
as school and ethnic group. Most of these attributes are categorical, thereby
we can think of exam mark prediction as a tensor completion problem
where each of the modes corresponds to a categorical attribute. In
particular, we have used the following attributes: school ($139$),
gender ($2$), VR-band ($3$), ethnic ($11$), and year ($3$), leading
to a $5$-order tensor $\boldsymbol{{\cal W}}\in\mathbb{R}^{139\times2\times3\times11\times3}$.

We have selected randomly $5\%$ of the instances to make the test
set and another $5\%$ of the instances for the validation set. From
the remaining instances, we have randomly chosen $m$ of them for
several values of $m$. This procedure has been repeated $20$ times
and the average performance is presented in Figure \ref{fig:RealResults} (Left). 
There is a distinguishable improvement of our approach with respect
to tensor trace norm regularization. To check whether this gap is significant, we have
conducted a set of paired $t$-tests for each value of $m$. In
all cases we obtained a $p$-value below $0.01$.

\begin{figure}
\begin{centering}
\includegraphics[width=0.49\columnwidth]{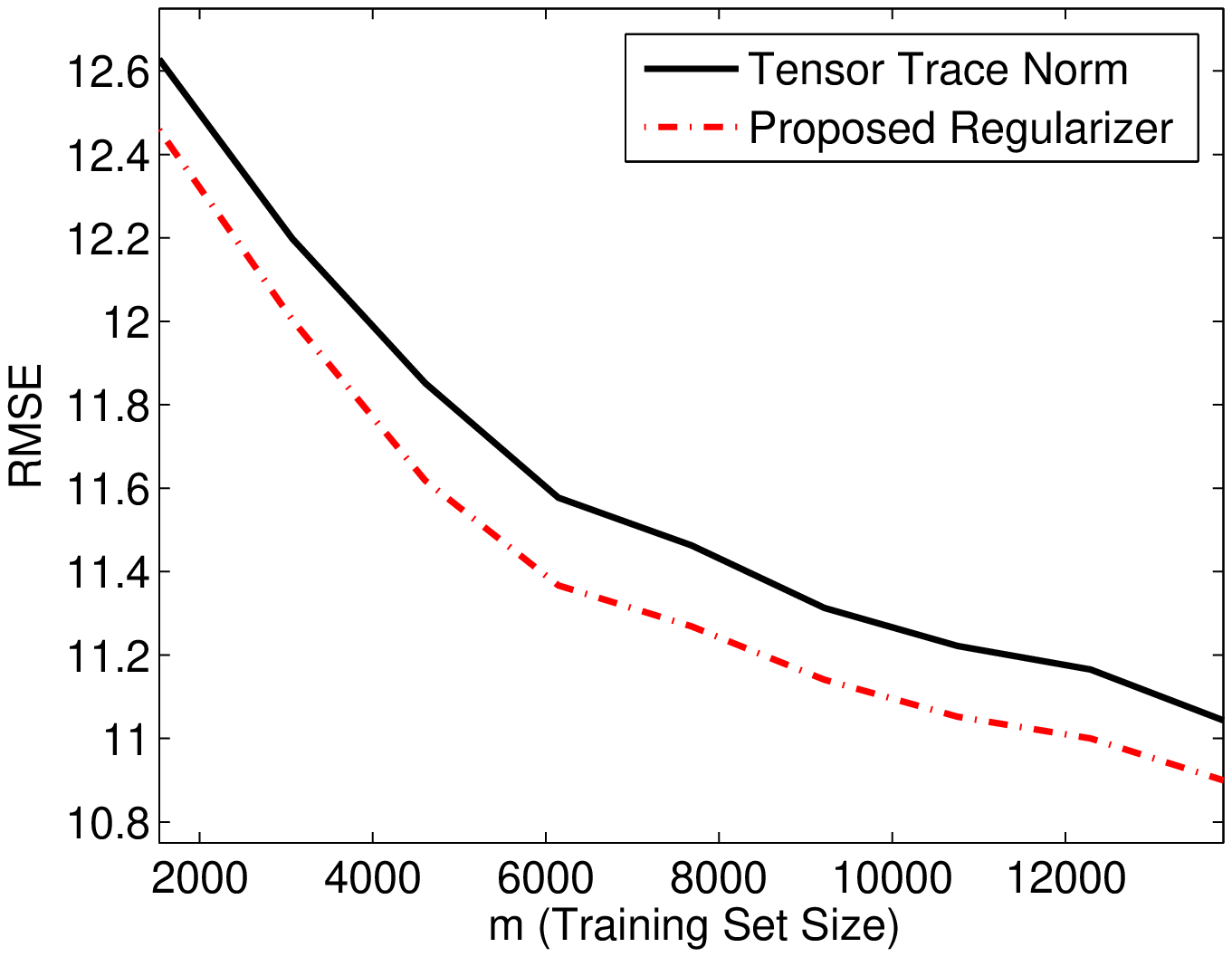}
\includegraphics[width=0.49\columnwidth]{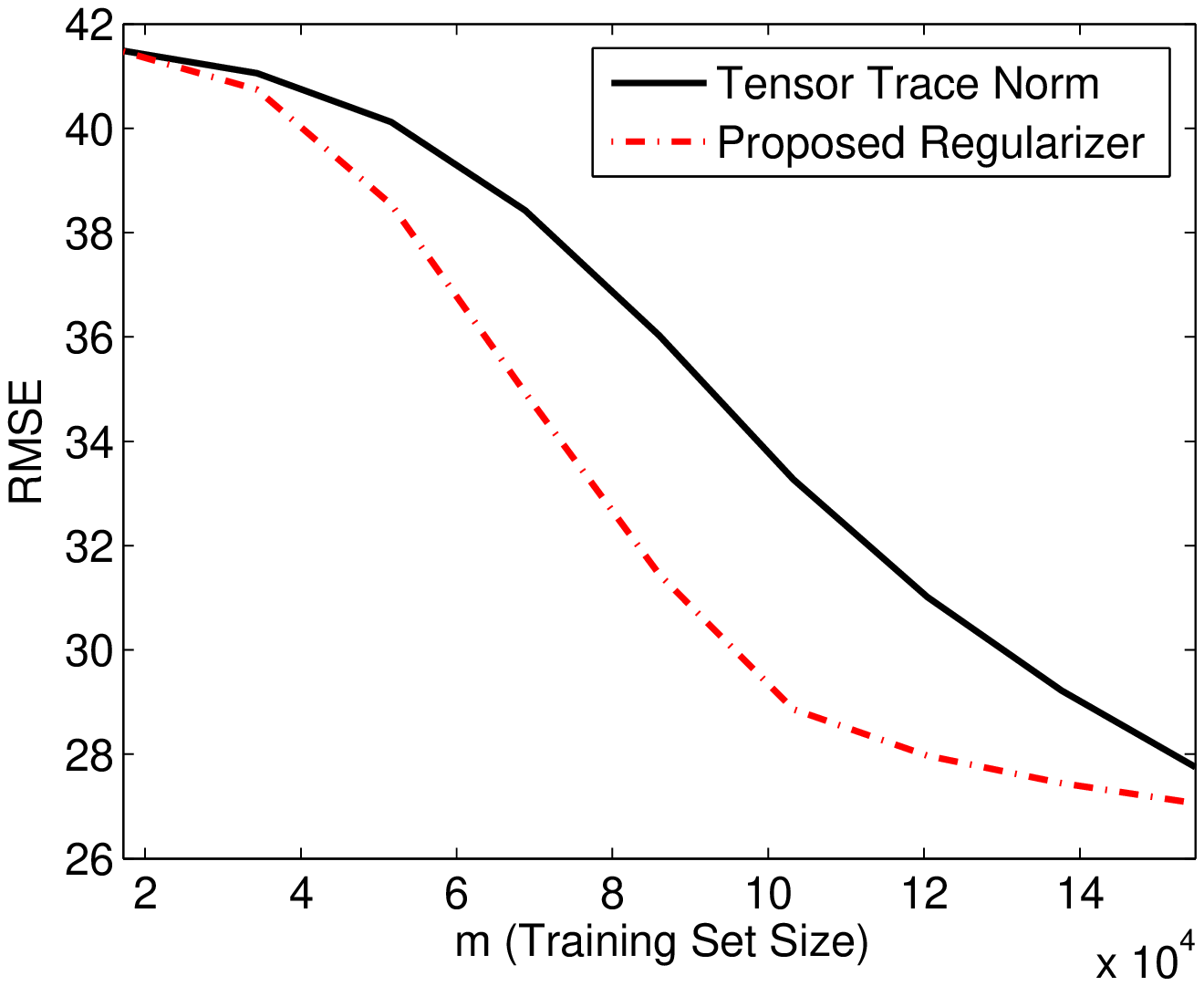}
\par\end{centering}

\caption{\label{fig:RealResults} Root Mean Squared Error of tensor trace norm and the proposed regularizer for ILEA dataset (Left) and Ocean video (Right).}
\end{figure}

\subsection{Video Completion}

In the second real-data experiment we have performed a video completion test.
Any video can be treated as a $4$-order tensor: ``width'' $\times$ ``height''
$\times$ ``RGB'' $\times$ ``video length'', so we can use tensor completion
algorithms to rebuild a video from a few inputs, a procedure that
can be useful for compression purposes. In our case, we have used
the Ocean video, available at \cite{Liu}. This video sequence can be
treated as a tensor $\boldsymbol{{\cal W}}\in\mathbb{R}^{160\times112\times3\times32}$.
We have randomly sampled $m$ tensors elements as training data, $5\%$ of them
as validation data, and the remaining ones composed the test set. After
repeating this procedure $10$ times, we present the average results
in Figure \ref{fig:RealResults} (Right). The proposed approach is noticeably better than the tensor trace norm
in this experiment. This apparent outcome is strongly supported by
the paired $t$-tests which we run for each value of $m$, obtaining
always $p$-values below $0.01$, and for the cases $m>5\times10^{4}$,
we obtained $p$-values below $10^{-6}$. 

\section{Conclusion}
\label{sec:conclusion}
In this paper, we proposed a convex relaxation for the average of the rank of the matricizations of a tensor. 
We compared this relaxation to a commonly used convex relaxation used in the context of tensor completion, which is based on the trace norm. We proved that this second relaxation is not tight and argued that the proposed convex regularizer may be advantageous. Empirical comparisons indicate that our method consistently improves in terms of estimation error over tensor trace norm regularization, while being computationally comparable on the range of problems we considered. In the future it would be interesting to study methods to speed up the computation of the proximity operator of our regularizer and investigate its utility in tensor learning problem beyond tensor completion such as multilinear multitask learning \cite{Bern}.

\section*{Appendix}
\appendix
In this appendix, we describe an auxiliary result and present the main steps for the computation of an approximate projection.

\section{A Useful Lemma}
\label{sec:useful_Lemma}

\begin{lemma}
Let $\C_1,\dots,\C_N$ be convex subsets of a Euclidean space and let ${\cal D}=\bigcap_{n=1}^N\C_n \neq \emptyset$. Let $g: \prod_{n=1}^N \C_n \rightarrow \R$ and let $h: {\cal D} \rightarrow \R$ be the function defined, for every $x \in {\cal D}$, as $h(x) = g(x,\dots,x)$. 
Then, for every $x \in {\cal D}$, it holds that
$$
h^{**}(x) \geq g^{**}(x_1,\dots,x_N)\left|_{x_n=x,~\forall n\in [N]}\right..
$$
\label{lem:1}
\end{lemma}
\begin{proof}
Since the restriction of $g$ on ${\cal D}^N \subseteq \prod_{n=1}^N \C_n$ equals to $h$, the convex envelope of $g$ when evaluated on the smaller set ${\cal D}^N$ cannot be larger than the convex envelope of $h$ on ${\cal  D}$.
\end{proof}
Using this result it is immediately possible to derive a convex lower bound for the function $R$ in equation \eqref{eq:rank}. Since the convex envelope of the rank function on the unit ball of the spectral norm is the trace norm, using Lemma \ref{lem:1} with $\C_n = \{\bW : \|W_{(n)}\|_\infty \leq 1\}$ 
and 
$$g(\bW_1,\dots,\bW_N) = \frac{1}{N}\sum_{n=1}^N \rank((W_n)_{(n)}),$$ we conclude that 
the convex envelope of the function $R$ on the set $\calG_{\Sp}$ is bounded from below by $\frac{1}{N}\sum_{n=1}^N \|W_{(n)}\|_{\rm tr}$. Likewise the convex envelope of $R$ on the set $\alpha \calG_{\Fr}$ is lower bounded by the function $\Omega_\alpha$ in equation \eqref{eq:our}.

\section{Computation of an Approximated Projection}
\label{sec:projection}

Here, we address the issue of computing an approximate Euclidean projection onto the set 
$$
\mathcal{S} = \{v \in \R^d: v_1 \geq \cdots \geq v_d \geq 0\}.
$$ 
That is, for every $v$, we shall find a point $\aP_{\mathcal{S}}(v) \in {\cal S}$ such that 
\begin{equation}
\left\Vert \aP_{\mathcal{S}}\left(v\right)-z\right\Vert _{2}\leq\left\Vert v-z\right\Vert _{2},\,\forall z\in\mathcal{S}\label{eq:approxProj}.
\end{equation}
As noted in \cite{Subgradient}, in order to build $\aP_{\mathcal{S}}$ such that this property holds true, it is useful to express the set of interest as the smallest one in a series of nested sets. In our problem, we can express $\mathcal{S}$ as

\begin{center}
$\mathcal{S}=\mathcal{S}_{d}\subseteq\mathcal{S}_{d-1}\subseteq\ldots\subseteq\mathcal{S}_{1}$,
\par\end{center}
where $\mathcal{S}_{i}:=\left\{ v\in\mathbb{R}^{d}\,:\, v_{1}\geq v_{2}\geq\ldots\geq v_{i}, v\geq0\right\} $.
This property allows us to sequentially compute an approximate projection on the set ${\cal S}$ using the formula
\beq
\aP_{\mathcal{S}}\left(v\right)=P_{\mathcal{S}_{d}}\left(P_{\mathcal{S}_{d-1}}\cdots\left(P_{\mathcal{S}_{1}}\left(v\right)\right)\right)
\label{eq:0000}
\eeq
where, for every close convex set ${\cal C}$, we let $P_{\cal C}$ be the associated projection operator. Indeed, following  \cite{Subgradient}, we can argue by induction on $i$ that $\aP_{\mathcal{S}}\left(v\right)$ verifies 
condition (\ref{eq:approxProj}). The base case is $\left\Vert P_{\mathcal{S}_{1}}\left(v\right)-z\right\Vert _{2}=\left\Vert v-z\right\Vert _{2}$,
which is obvious. Now, if for a given $1 \leq i \leq d-1$ it holds that
\begin{center}
$\left\Vert P_{\mathcal{S}_{i}}\left(\cdots P_{\mathcal{S}_{1}}\left(v\right)\right)-z\right\Vert _{2}\leq\left\Vert v-z\right\Vert _{2}$
\par\end{center}
then
\begin{center}
$\left\Vert P_{\mathcal{S}_{i+1}}\left(P_{\mathcal{S}_{i}}\left(\cdots P_{\mathcal{S}_{1}}\left(v\right)\right)\right)-z\right\Vert _{2}\leq\left\Vert P_{\mathcal{S}_{i}}\left(\cdots P_{\mathcal{S}_{1}}\left(v\right)\right)-z\right\Vert _{2}\leq\left\Vert v-z\right\Vert _{2}$, 
\par\end{center}
since $z$ is also contained in $\mathcal{S}_{i+1}$. 

Note that to evaluate the right hand side of equation \eqref{eq:0000} we do not require full knowledge of $P_{{\cal S}_i}$, we only need to compute $P_{\mathcal{S}_{i+1}}(v)$ for $v\in\mathcal{S}_{i}$. The next proposition describes a recursive formula to achieve this step.

\begin{proposition}
For any $v\in\mathcal{S}_{i}$, we express its first $i$ elements
as $v_{1:i}=\left[v_{1:i-j},\,v_i\boldsymbol{1}^{j}\right]$, where
the last $j \in[i]$ is the largest integer such that $v_{i-j+1}= v_{i-j+2} = \cdots =v_i$. It holds that
\begin{equation}
P_{\mathcal{S}_{i+1}}\hspace{-.1truecm}\left(v\right)=\begin{cases}
v & {\rm if}\,v_i\geq v_{i+1}\\
\left[v_{1:i-j},\,\left(v_i+\frac{v_{i+1}-v_i}{j+1}\right)\boldsymbol{1}^{j+1}, \, v_{i+2:d}\right] & {\rm if}\,v_i<v_{i+1}~{\rm and}~
 v_{i-j}\geq v_i\hspace{-.03truecm}+\hspace{-.03truecm}\frac{v_{i+1}-v_i}{j+1}\\
P_{\mathcal{S}_{i+1}}\hspace{-.1truecm}\left(\hspace{-.03truecm}\left[v_{1:i-j},\, v_{i-j}\boldsymbol{1}^{j},\, v_{i+1}\hspace{-.03truecm}-\left(v_{i-j}\hspace{-.03truecm}-\hspace{-.03truecm}v_i\right)j, \, v_{i+2:d}\right]\hspace{-.03truecm}\right) & {\rm otherwise},
\nonumber
\end{cases}
\end{equation}
where $\boldsymbol{1}^d\in \mathbb{R}^d$ denotes the vector containing $1$ in all its elements.
\label{prop:projS}
\end{proposition}
\begin{proof}
The first case is straightforward. In the following we prove the
remaining two. In both cases it will be useful to recall that the
projection operator $P_{\mathcal{C}}$ on any convex set $\mathcal{C}$
is characterized as 

\begin{equation}
x=P_{\mathcal{C}}\left(y\right)\Longleftrightarrow\left\langle y-x,\, z-x\right\rangle \leq0, ~ \forall z\in\mathcal{C}.\label{eq:proOperatorProp}
\end{equation}

\begin{algorithm}[t]
\begin{algorithmic}
\STATE \textbf{Input}: $y\in \mathbb{R}_+^d$.
\STATE \textbf{Output}: $v \in \mathcal{S}$.
\STATE \textbf{Initialization}: $v\leftarrow y$.

\FOR {$i=1,2,\dots, d$}
\WHILE {$v_{i}<v_{i+1}$}

\STATE $j \leftarrow$ ${\rm argmax} \{\ell:  \ell \in [i], v_{i}=v_{i-\ell+1}\}$ 

\IF {$v_i+\frac{v_{i+1}-v_i}{j+1}$}
\STATE $v_{1:i+1} \leftarrow \left[v_{1:i-j},\,\left(v_i+\frac{v_{i+1}-v_i}{j+1}\right)\boldsymbol{1}^{j+1}\right] $
\ELSE
\STATE $v_{1:i+1} \leftarrow \left[v_{1:i-j},\, v_{i-j}\boldsymbol{1}^{j},\, v_{i+1}-\left(v_{i-j}-v_i\right)j\right]$
\ENDIF

\ENDWHILE
\ENDFOR
\end{algorithmic}
\caption{Computing an approximated projection 
onto the set $\mathcal{S} = \{v \in \R^d: v_1 \geq \cdots \geq v_d \geq 0\}$.
}
\label{alg:proj}
\end{algorithm}

To prove the second case, we use property (\ref{eq:proOperatorProp}) and apply simple
algebraic transformations to obtain, for all $z\in\mathcal{S}_{i+1}$, that
$$
\left\langle v-P_{\mathcal{S}_{i+1}}\left(v\right),\, z-P_{\mathcal{S}_{i+1}}\left(v\right)\right\rangle = 
\frac{v_{i+1}-v_i}{j+1}\left(jz_{i+1}-\left\Vert z_{i-j+1:i}\right\Vert _{1}\right)\leq0.
$$

Finally we prove the third case. We want to show that if  $x=P_{\mathcal{S}_{i+1}}(v)$ then
$$
x=P_{\mathcal{S}_{i+1}}\left(\left[v_{1:i-j},\, v_{i-j}\boldsymbol{1}^{j},\, v_{i+1}-\left(v_{i-j}-v_i\right)j, \, v_{i+2:d}\right]\right).
$$ 

By using property (\ref{eq:proOperatorProp}), the last equation is equivalent to the statement that if
\begin{equation}
\left\langle v-x,\, z-x\right\rangle \leq0,\,\,\forall z\in\mathcal{S}_{i+1}~~{\rm then}\label{eq:leftImpl}
\end{equation}
\begin{equation}
\left\langle \left[v_{1:i-j},\, v_{i-j}\boldsymbol{1}^{j},\, v_{i+1}-\left(v_{i-j}-v_i\right)j, \, v_{i+2:d}\right]-x,\, z-x\right\rangle \leq0,\,\,\forall z\in\mathcal{S}_{i+1}.
\label{eq:rightImpl}
\end{equation}

A way to show that it holds true is to prove that the term in the left hand side of (\ref{eq:rightImpl}) is upper bounded by the corresponding term  in (\ref{eq:leftImpl}). That is, for every $z\in\mathcal{S}_{i+1}$, we want to show that
$$
\left\langle \left[v_{1:i-j},\, v_{i-j}\boldsymbol{1}^{j},\, v_{i+1}-\left(v_{i-j}-v_i\right)j, \, v_{i+2:d}\right]-v,z-x\right\rangle \leq0.
$$ 
A direct computation yields the equivalent inequality
\begin{equation}
\left(v_{i-j}-v_i\right)\left(jx_{i+1}-\left\Vert x_{i-j+1:i}\right\Vert _{1}+\left\Vert z_{i-j+1:i}\right\Vert _{1}-jz_{i+1}\right)\leq 0.
\label{eq:appProj1}
\end{equation}

Since $x=P_{\mathcal{S}_{i+1}}\left(v\right)$, $v_{i-j+1} = v_{i-j+2} = \cdots = v_i$ and $v_{i+1}>v_i$, then $x_{i-j+1}=x_{i-j+2}=\cdots=x_{i+1}$. Consequently, the left hand side of inequality (\ref{eq:appProj1}) is equivalent to
$$
\left(v_{i-j}-v_i\right)\left(\left\Vert z_{i-j+1:i}\right\Vert _{1}-jz_{i+1}\right)\leq 0.
$$
Note that the first factor is negative and the second is positive because $z$ and $v$ are in $\mathcal{S}_{i+1}$. The result follows.
\end{proof}

Algorithm \ref{alg:proj} summarizes our method to compute the approximated projection operator onto the set $\mathcal{S}$, based on Proposition \ref{prop:projS}.

\end{document}